\documentclass{article}

\usepackage[utf8]{inputenc} %
\usepackage[T1]{fontenc}    %

\usepackage[bitstream-charter,cal=cmcal]{mathdesign}
 
\usepackage{amsmath}
\usepackage[scaled=0.92]{PTSans}

\usepackage[
  paper  = letterpaper,
  left   = 1.65in,
  right  = 1.65in,
  top    = 1.0in,
  bottom = 1.0in,
  ]{geometry}

\usepackage[usenames,dvipsnames,table]{xcolor}
\definecolor{shadecolor}{gray}{0.9}

\usepackage[final,expansion=alltext]{microtype}
\usepackage[english]{babel}
\usepackage[parfill]{parskip}
\usepackage{afterpage}
\usepackage{framed}

{\endMakeFramed}

\usepackage{lineno}

\usepackage{ragged2e}

\newcounter{parcount}

\usepackage{fancyvrb}
\fvset{fontsize=\normalsize}

\usepackage[acronym,smallcaps,nowarn]{glossaries}

\usepackage{natbib}
\usepackage[colorlinks,linktoc=all]{hyperref}
\usepackage[all]{hypcap}
\hypersetup{citecolor=blue}
\hypersetup{urlcolor=blue}
\hypersetup{linkcolor=Black}

\usepackage[nameinlink]{cleveref}

\usepackage{natbib}
\usepackage{hyperref}       %
\usepackage{url}            %
\usepackage{booktabs}       %
\usepackage{nicefrac}       %
\usepackage{microtype}      %
\usepackage{amsthm}
\usepackage{shortcuts}
\usepackage{tikz}
\usepackage{thmtools}
\usepackage{thm-restate}
\usepackage{subfig,caption,graphicx}
\usepackage[ruled,vlined]{algorithm2e}
\usepackage{listings}
\usepackage{rotating,enumitem}

\newcount\Comments  %
\definecolor{darkgreen}{rgb}{0,0.5,0}
\definecolor{darkred}{rgb}{0.7,0,0}
\definecolor{teal}{rgb}{0.3,0.8,0.8}
\definecolor{blue}{rgb}{0,0,1}
\definecolor{purple}{rgb}{0.5,0,1}
\newcommand{\kibitz}[2]{\ifnum\Comments=1\textcolor{#1}{#2}\fi}

\newcommand{\ci}{\perp\!\!\!\perp}
\theoremstyle{plain}
\newtheorem{theorem}{Theorem}
\newtheorem{proposition}[theorem]{Proposition}
\newtheorem{lemma}[theorem]{Lemma}
\newtheorem{corollary}[theorem]{Corollary}
\theoremstyle{definition}

\newtheorem{definition}{Definition}

\usetikzlibrary{shapes,decorations,arrows,calc,arrows.meta,fit,positioning}
\tikzset{
    -Latex,auto,node distance =1 cm and 1 cm,semithick,
    state/.style ={ellipse, draw, minimum width = 0.7 cm},
    point/.style = {circle, draw, inner sep=0.04cm,fill,node contents={}},
    square/.style = {rectangle, draw, inner sep=0.04cm,fill,node contents={}},
    bidirected/.style={Latex-Latex,dashed},
    el/.style = {inner sep=2pt, align=left, sloped}
}

\usepackage{authblk}
\title{I-SPEC: An End-to-End Framework for Learning Transportable, Shift-Stable Models}
\author[1]{Adarsh Subbaswamy}
\author[1]{Suchi Saria}
\affil[1]{Department of Computer Science; Johns Hopkins University}
\date{}                     %

\begin{document}

\maketitle

\begin{abstract}
Shifts in environment between development and deployment cause classical supervised learning to produce models that fail to generalize well to new target distributions. Recently, many solutions which find invariant predictive distributions have been developed. Among these, graph-based approaches do not require data from the target environment and can capture more stable information than alternative methods which find stable feature sets. However, these approaches assume that the data generating process is known in the form of a full causal graph, which is generally not the case. In this paper, we propose \textsc{I-Spec}, an end-to-end framework that addresses this shortcoming by using data to learn a partial ancestral graph (PAG). Using the PAG we develop an algorithm that determines an interventional distribution that is stable to the declared shifts; this subsumes existing approaches which find stable feature sets that are less accurate. We apply \textsc{I-Spec} to a mortality prediction problem to show it can learn a model that is robust to shifts without needing upfront knowledge of the full causal DAG.

 \end{abstract}

\section{Introduction}
One of the primary barriers to the deployment of machine learning models in safety-critical applications is unintended behaviors arising at deployment that were not problematic during model development. For example, predictive policing systems have been shown to be vulnerable to predictive feedback loops that cause them to disproportionately overpatrol certain neighborhoods \citep{lum2016predict,ensign2018runaway}, and a patient triage model erroneously learned that asthma lowered the risk of mortality in pneumonia patients \citep{caruana2015intelligible}. At the heart of many such unintended behaviors are \emph{shifts in environment}---changes in the conditions that generated the training data and deployment data \citep{subbaswamy2019development}. An important step for ensuring that models will perform reliably under shifting conditions is for model developers to anticipate failures and train models in a way that addresses likely sources of error.

Consider the study of \citet{zech2018variable}, who trained a model to diagnose pneumonia from chest X-rays using data from one hospital. Notably, the X-rays contained stylistic features, including inlaid tokens that encoded geometric information such as the X-ray orientation. When they applied the model to data from other hospital locations, they found the model's performance significantly deteriorated, indicating that it failed to generalize across the shifts between hospitals. In particular, shifts in the distribution of style features occurred due to differences in equipment at different hospital locations and differences in imaging protocols between hospital departments.

More formally, \citeauthor{zech2018variable} encountered an instance of \emph{dataset shift}, in which shifts in environment resulted in differing train and test distributions. Typical solutions for addressing dataset shift use samples from the test distribution to reweight training samples during learning (see \citet{quionero2009dataset} for an overview). In many practical applications, however, there are unknown or multiple possible test environments (e.g., for a cloud-based machine learning service), making it infeasible to acquire test samples. In contrast with reweighting solutions, \emph{proactive} solutions do not use test samples during learning. For example, one class of proactive solutions is entirely \emph{dataset-driven}: using datasets from multiple training environments, they empirically determine a \emph{stable conditional distribution} that is invariant across the datasets (e.g., \citet{muandet2013domain,rojas2018invariant,arjovsky2019invariant}). A model of this distribution is then used to make predictions in new, unseen environments.

While the dataset-driven methods find a distribution that is invariant across the training datasets, they do not, in general, provide guarantees about the specific shifts in environment to which the resulting models are stable. This information is crucially important in safety-critical domains where incorrect decision making can lead to failures. In the pneumonia example, suppose we had multiple training datasets which contained shifts in style features due to differing equipment, but, critically, did not contain shifts in protocols between departments. When a dataset-driven method finds a predictive distribution that is invariant across the training datasets, its developers will not know that this distribution is stable to shifts in equipment but is not stable to shifts in imaging protocols. When the resulting model is deployed at a hospital with different imaging protocols (e.g., distribution of front-to-back vs back-to-front X-rays), the model will make (potentially arbitrarily) incorrect predictions resulting in unanticipated misdiagnoses and disastrous failures. 

Alternative methods use graphical representations of the data generating process (DGP) (e.g., causal \emph{directed acylic graphs} (DAGs)), letting developers proactively reason about the DGP to specify shifts and provide stability guarantees. One advantage of explicit graph-based methods is that they allow the computation of stable \emph{interventional} \citep{subbaswamy2019preventing} and \emph{counterfactual} distributions \citep{subbaswamy2018counterfactual}; these retain more stable information than conditional distributions, leading to higher accuracy. A primary challenge in applying these approaches, however, is that in large-scale complex domains it is very difficult to fully specify the graph (i.e., edge adjacencies and directions) from prior knowledge alone. We address this by extending graphical methods for finding stable distributions to partial graphs that can be learned directly from data.

\textbf{Contributions:}
A key impediment to the deployment of
machine learning is the lack of methods for training models that can generalize despite shifts across training and test environments. Stable interventional distributions estimated from data yield models that are guaranteed to be invariant to shifts (e.g., the modeler can upfront identify which shifts the model is protected against). However, to estimate such distributions, prior approaches require knowledge of the underlying causal DAG or extensive samples from multiple training environments. We propose \textsc{I-Spec}, a novel end-to-end framework which allows us to estimate stable interventional distributions when we do not have prior causal knowledge of the full graph. To do so, we learn a partial ancestral graph (PAG) from data; the PAG captures uncertainty in the graph structure. Then, we use the PAG to inform the choice of mutable variables, or shifts to protect against. We develop an algorithm that uses the PAG and set of mutable variables to determine a stable interventional distribution. We prove the soundness of the algorithm and prove that it subsumes existing dataset-driven approaches which find stable conditional distributions. Empirically, we apply \textsc{I-Spec} to a large, complicated healthcare problem and show that we are able to learn a PAG, use it to inform the choice of mutable variables, and learn models that generalize well to new environments. We also use simulated data to provide insight into when stable models are desirable by examining how shifts of varying magnitude affect the difference in performance between stable and unstable models.

\section{Background}\label{sec:prelims}
The proposed framework, \textsc{I-Spec}, uses PAGs and interventional distributions, which we briefly overview here.

\textbf{Notation:} Sets of variables are denoted by bold capital letters while their assignments are denoted by bold lowercase letters. The  sets  of  parents, children, ancestors, and descendants in a graph $\mathcal{G}$ will be denoted by $pa(\cdot)$, $ch(\cdot)$, $an(\cdot)$, and $de(\cdot)$, respectively. We will consider prediction problems with observed variables $\mathbf{O}$ and target variable $Y$. 

\textbf{Causal Graphs:} We assume the DGP underlying a prediction problem can be represented as a causal DAG with latent variables, or equivalently, a causal \emph{acylic directed mixed graph} (ADMG) over $\mathbf{O}$ which contains directed ($\rightarrow$) and bidirected ($\leftrightarrow$, representing unobserved confounding) edges. A causal \emph{mechanism} is the functional relationship that generates a child from its (possibly unobserved) parents. 

Multiple ADMGs can contain the same conditional independence and ancestral information about the observed variables $\mathbf{O}$---consider, for example, an ADMG $\mathcal{G}_1$ with edges $X\rightarrow Y, X\leftrightarrow Y$ and an ADMG $\mathcal{G}_2$ with edge $X\rightarrow Y$. However, an ADMG is associated with a unique \emph{maximal ancestral graph} (MAG) \citep{richardson2002ancestral} which represents a set of ADMGs that share this information, with at most one edge between any pair of variables (e.g., the MAG $\mathcal{M}$ associated with $\mathcal{G}_1$ and $\mathcal{G}_2$ is $X\rightarrow Y$). The problem is that multiple MAGs may contain the same independences (e.g., $X\rightarrow Y$ and $X\leftrightarrow Y$ are \emph{Markov equivalent}). Fortunately, a \emph{partial ancestral graph} (PAG) $\mathcal{P}$ represents an equivalence class of MAGs, denoted by $\{\mathcal{P}\}$. $\mathcal{P}$ and every MAG in $\{\mathcal{P}\}$ have the same adjacencies, but they differ in the edge marks. An arrow head (or tail) is present in $\mathcal{P}$ if that head (or tail) is present in all MAGs in $\{\mathcal{P}\}$. Otherwise, the edge mark is $\circ$ and the edge is partially (or non) directed. The PAG for $\mathcal{G}_1$, $\mathcal{G}_2$, and $\mathcal{M}$ is $X \circcirc Y$. PAGs can be learned from data, and are the output of the FCI algorithm \citep{spirtes2000causation}.

Because PAGs are partial graphs, we require a few additional definitions. First, a path from $X$ to $Y$ is \emph{possibly directed} if no arrowhead along the path is directed towards $X$. In such a path, $X\in PossibleAn(Y)$ is a \emph{possible ancestor} of $Y$, and $Y\in PossibleDe(X)$ is a \emph{possible descendant} of $X$. There are two kinds of directed edges in MAGs and PAGs. A directed edge $X\rightarrow Y$ is \emph{visible} if there is a node $Z$ not adjacent to $Y$, such that either there is an edge between $Z$ and $X$ that is into $X$, or there is a collider path between $Z$ and $X$ that is into $X$, and every node on the path is a parent of $Y$ \citep{maathuis2015generalized}. Otherwise, the edge is \emph{invisible}. The importance of visible edges is that if $X\rightarrow Y$ is visible, then it implies that there is no unobserved confounder (i.e., no $X \leftrightarrow Y$ edge in any ADMG in the equivalence class).

\textbf{Interventional Distributions:} We now review interventional distributions, which we use to make stable predictions. First, note that the distribution of observed variables $\mathbf{O}$ in an ADMG factorizes as 
\begin{equation}\label{eq:admg}
    P(\mathbf{O}) = \sum_{\mathbf{U}} \prod_{O_i \in \mathbf{O}} P(O_i|pa(O_i)) P(\mathbf{U}),
\end{equation}
where $\mathbf{U}$ are unobserved variables corresponding to the bidirected edges. An interventional distribution of the form $P(Y|Z, do(X))$ is defined in terms of the $do$ operator \citep{pearl2009causality}.\footnote{We will use $P(Y|Z,do(X))$ and $P_X(Y|Z)$ interchangeably.} Graphically, in ADMGs the intervention $do(X)$ deletes all edges into $X$. Distributionally, $do(\mathbf{X}=\mathbf{x})$ deletes the $P(X|pa(X))$,$\forall X\in\mathbf{X}$ terms in (\ref{eq:admg}) to yield the interventional distribution.
The difficulty of using interventional distributions is that they are not always \emph{identifiable} as a function of the observational data distribution.

\begin{definition}[Causal Identifiability]
For disjoint $\mathbf{X},\mathbf{Y} \subseteq \mathbf{O}$, the effect of an intervention $do(\mathbf{x})$ on $\mathbf{Y}$ conditional on $\mathbf{Z}$ is said to be identifiable from $P$ in $\mathcal{G}$ if $P_{\mathbf{x}}(\mathbf{Y}|\mathbf{Z})$ is (uniquely) computable from $P(\mathbf{O})$ in any causal model which induces $\mathcal{G}$.
\end{definition}

The ID algorithm \citep{shpitser2006identification,shpitser2006idc} takes disjoint sets $\mathbf{X},\mathbf{Y},\mathbf{Z}\subset\mathbf{O}$ and an ADMG $\mathcal{G}$ and returns an expression in terms of $P(\mathbf{O})$ if $P_\mathbf{X}(\mathbf{Y}|\mathbf{Z})$ is identifiable. Recently, the ID algorithm was extended to PAGs \citep{jaber2019causal}, with the CIDP algorithm\footnote{We restate this algorithm in Appendix \ref{app:CIDP}} \citep{jaber2019idc} returning an expression for $P_\mathbf{X}(\mathbf{Y}|\mathbf{Z})$ if the conditional interventional distribution is uniquely computable (i.e., identifiable) from $P(\mathbf{O})$ for all ADMGs in the equivalence class $\{\mathcal{P}\}$ defined by a PAG $\mathcal{P}$. We use CIDP in the proposed method to determine identified interventional expressions.

\section{Methods}
We now present \textsc{I-Spec}, a framework for finding a stable distribution that is invariant to shifts in environments. \textsc{I-Spec} works as follows: Given datasets collected from multiple source environments, a user defines a graphical \emph{invariance specification} by first learning a PAG from the combined datasets (without requiring prior causal knowledge). Then, the user can determine which shifts to protect against by reasoning about the PAG and consulting it regarding shifts that occurred across the datasets. Given the resulting invariance specification (i.e., PAG and shifts to protect against), graphical criteria are used to search for the best-performing stable interventional distribution which is guaranteed to be invariant to the specified shifts. 

The rest of this section is organized as follows:
In Section \ref{subsec:transport} we introduce invariance specifications. Next, in Section \ref{subsec:framework} we describe the steps of \textsc{I-Spec} (Algorithm \ref{alg}) and prove its correctness. Then, in Section \ref{subsec:connections} we establish the superiority of stable interventional distributions over stable conditional distributions, proving that Algorithm \ref{alg} subsumes existing dataset-driven methods. Finally, in Section \ref{subsec:one} we discuss how \textsc{I-Spec} can be adapted to settings in which data from only one environment is available.

\subsection{Graphical Invariance Specifications}\label{subsec:transport}
Our goal is to predict accurately in new environments without using test samples. To do so, we need a way to represent the possible environments and how they can differ. For this reason, we will now introduce invariance specs, which are built around a PAG and specify shifts to protect against. They do not require prior causal knowledge and can be learned from data. Given the invariance spec, we show that certain interventional distributions provide stability guarantees to the specified shifts.

First, we formalize the notion of a \emph{stable distribution}. Stable distributions are the same in all environments, can be learned from the source environment data, and can be applied in the target environment without any adjustment.
\begin{definition}[Stable Distribution]
For environment set $\mathcal{E}$, a distribution $P(Y|\mathbf{Z}),\mathbf{Z}\subseteq\mathbf{O}$ is said to be stable if, for any two environments in $\mathcal{E}$ corresponding to joint distributions $P^{(1)}(\mathbf{O})$ and $P^{(2)}(\mathbf{O})$, $P^{(1)}(Y|\mathbf{Z})=P^{(2)}(Y|\mathbf{Z})$.
\end{definition}

Stable distributions are defined with respect to a set of environments. We develop invariance specifications as a way to represent a set of environments when we do not have prior causal knowledge by generalizing \emph{selection diagrams} \citep{pearl2011transportability,subbaswamy2019preventing}, a representation that assumes a known causal graph.

\begin{definition}[Selection Diagram]
A selection diagram is an ADMG augmented with auxiliary selection variables $\mathbf{S}$ such that for $S\in\mathbf{S},X\in\mathbf{O}$ an edge $S\rightarrow X$ denotes the mechanism that generates X can vary across environments. Selection variables may have at most one child.
\end{definition}

A selection diagram describes a set of environments whose DGPs share the same underlying graph structure (i.e., ADMG). Only the causal mechanisms associated with the children of selection variables may differ across environments, usually expressed as distributional shifts in the $P(V|pa(V)),V\in ch(\mathbf{S})$ terms of the factorization of the joint $P(\mathbf{O})$ via Equation (\ref{eq:admg}).\footnote{This special case of shifts in mechanism are sometimes referred to as ``soft interventions''.}

Selection diagrams assume that both the full graph (i.e., ADMG) and the shifts (i.e., placement of selection variables) are known, prohibiting their use in complex domains. A natural idea to relax this would be to define a selection PAG, thus allowing for uncertainty in the graphical structure. However, a PAG augmented with selection variables would not technically be a PAG---for example, selection variables could not be used to determine visible edges in the PAG. For this reason, we introduce the notion of a graphical \emph{invariance specification} (or simply invariance spec) which generalizes selection diagrams.

\begin{definition}[Invariance spec]
An invariance spec is a 2-tuple $\langle\mathcal{P}, \mathbf{M}\rangle$ consisting of a graphical representation, $\mathcal{P}$, of the DGP and a set of mutable variables, $\mathbf{M}$, whose causal mechanisms are vulnerable to shifts.
\end{definition}

When $\mathcal{P}$ is a PAG, an invariance spec defines a set of environments which share the same underlying graph structure (i.e., ADMG) that is only known up to an equivalence class (the PAG). Now we say the mechanism shifts are associated with the \emph{mutable} variables $\mathbf{M}$ \citep{subbaswamy2019preventing}. Note that if $\mathcal{P}$ is an ADMG, then by augmenting $\mathcal{P}$ with selection variables $\mathbf{S}$ as parents of $\mathbf{M}$ we recover a selection diagram. We will only consider when $\mathcal{P}$ is a PAG, which means the invariance spec can be learned from data.

Like selection diagrams, invariance specs provide graphical criteria for determining if a distribution is stable. The next result states that a distribution is stable in an invariance spec $\langle \mathcal{P}, \mathbf{M} \rangle$ when the distribution is stable in every selection diagram corresponding to the equivalence class $\{\mathcal{P}\}$.\footnote{Proofs of all results are in Appendix \ref{app:proofs}.}

\begin{proposition}\label{prop:admiss}
Given an invariance spec $\langle \mathcal{P}, \mathbf{M} \rangle$, a distribution $P(Y|\mathbf{X})$ is stable if $Y\ci \mathbf{S}| \mathbf{X}$ in every ADMG in the equivalence class $\{\mathcal{P}\}$ augmented with selection variables $\mathbf{S}$ as parents of $\mathbf{M}$.
\end{proposition}

Armed with this graphical stability criterion, we now extend a prior result which showed that intervening on $\mathbf{M}$ yields a stable distribution in ADMGs \citep[Proposition 1]{subbaswamy2019preventing}. The extension to PAGs is what permits the use of interventional distributions to make stable predictions across the environments represented by an invariance spec, and is thus key to the correctness of \textsc{I-Spec}.

\begin{proposition}\label{prop:do-stability}
For invariance spec  $\langle \mathcal{P}, \mathbf{M} \rangle$, $\mathbf{Z}\subseteq{O}\setminus\mathbf{M}$, $P(Y|do(\mathbf{M}),\mathbf{Z})$ is stable to shifts in the mechanisms of $\mathbf{M}$.
\end{proposition}

\subsection{\textsc{I-SPEC} Step by Step}\label{subsec:framework}
\newlength{\textfloatsepsave} \setlength{\textfloatsepsave}{\textfloatsep} \setlength{\textfloatsep}{0pt}
\begin{algorithm}[!t]
 \LinesNumbered
 \SetKwInOut{Input}{input}\SetKwInOut{Output}{output}

 \Input{Datasets $\{\mathbf{D}\}$, Observed vars $\mathbf{O}$, Target $Y\in\mathbf{O}$, Environment Indicator $E \not\in\mathbf{O}$}
 \Output{Best stable interventional distribution found or \texttt{FAIL}.}
 Learn invariance spec structure, PAG $\mathcal{P}$ over $\mathbf{O}\cup \{E\}$\;
 Declare mutable variables $\mathbf{M}\subseteq \mathbf{O}$\;
 
 Let $Stable= [\,]$\;
 \For{$\mathbf{Z}\in\mathscr{P}(\mathbf{O}\setminus\{Y\})$}{
 \If{$\mathbf{X}=\mathbf{M}$, $\mathbf{Y}=\{Y\}$, $\mathbf{Z}=\mathbf{Z}$ satisfy \citet[Thm 30]{zhang2008causal}}{
        $Stable =$ append$(Stable, P_\mathbf{M}(Y|\mathbf{Z})= P(Y|\mathbf{W})$\;
    }
 \ElseIf{CIDP$(\mathbf{M},\{Y\}, \mathbf{Z}\setminus\mathbf{M})\not = $\texttt{FAIL}}{
    $Stable =$ append$(Stable, P_\mathbf{M}(Y|\mathbf{Z}\setminus\mathbf{M}))$
 }
 }
 \If{$Stable = [\,]$}{\KwRet \texttt{FAIL}\;}
 \KwRet Dist. in $Stable$ with lowest validation loss\;
 \caption{\textsc{I-Spec}}
 \label{alg}
\end{algorithm}

We have just defined invariance specs and established that intervening on mutable variables yields a stable distribution. We are now ready to discuss each step of \textsc{I-Spec} (Alg \ref{alg}).

\textbf{1) Learning Invariance Spec Structure $\mathcal{P}$:}
The first step (Line 1) in creating an invariance spec is to learn the graphical representation of the DGP. We will assume \emph{faithfulness}: the independences implied by the graph are the only independences in the data distribution. Consider the DGP represented by the ADMG $\mathcal{G}$ in Fig \ref{fig:subgraph}a, in which the observed variables are $\mathbf{O}=\{X_1, X_2, X_3,Y\}$ and the goal is to predict $Y$. While every environment $E$ (e.g., location) shares this graph (including unseen environments), certain mechanisms corresponding to $ch(E)$ may vary across environments (e.g., $P(X_1|E)$). If we knew this full graph $\mathcal{G}$, then we could use existing graphical methods (e.g., \citet{subbaswamy2019preventing}) to find a stable distribution. However, in practice we usually do not know the full graph.

Instead, we will learn the structure of the DGP from datasets $\mathbf{D}$, containing observations of $\mathbf{O}$ collected in training environments $E$. While this problem has itself been extensively studied (as we discuss in Section \ref{sec:rel-work}; Related Work), we will use a simple extension of FCI \citep{spirtes2000causation}, a constraint-based structure learning algorithm which learns a PAG over the observed variables. FCI uses conditional independence tests to determine adjacencies and create a graph skeleton, and then uses a set of orientation rules to determine where to place edge marks \citep{zhang2008completeness}. We apply FCI by pooling the datasets $\mathbf{D}$, adding the environment indicator $E$ as a variable, and adding the logical constraint that $E$ causally precedes all variables in $\mathbf{O}$ (i.e., there can be no $V \rightarrow E$ edge for $V\in\mathbf{O}$).\footnote{Such prior knowledge can be specified in the Tetrad implementation of FCI \url{http://www.phil.cmu.edu/tetrad/}.}

Suppose we had datasets $\mathbf{D}$ generated from multiple environments according to the DGP in Fig \ref{fig:subgraph}a. Using the pooled FCI variant we described, we would learn the PAG $\mathcal{P}$ in Fig \ref{fig:subgraph}b. The PAG represents an equivalence class of ADMGs (which includes the one in Fig \ref{fig:subgraph}a). The $\circ$ edge marks denote structural uncertainty: there is an ADMG in the equivalence class in which this mark is an arrowhead, and an ADMG in the equivalence class in which this mark is a tail. Despite being a partial graph, the PAG still helps inform decisions about which shifts to protect against as discussed next.

\setlength{\textfloatsep}{\textfloatsepsave}
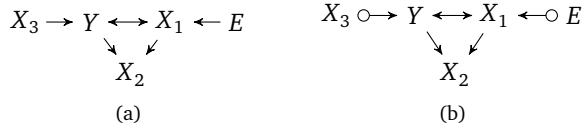
\begin{figure}[!t]
\vspace{-0.1in}
  \centering
  \subfloat[]{
    \begin{tikzpicture}[>=stealth',shorten >=1pt,auto,node distance=2cm,main node/.style={minimum size=0.6cm,font=\sffamily\Large\bfseries},scale=0.7,transform shape]

    \node[main node] (Y) at (0,0) {$Y$};
    \node[main node] (X1) at (1.5, 0) {$X_1$};
    \node[main node] (X2) at (0.75, -1) {$X_2$};
    \node[main node] (X3) at (-1.25, 0) {$X_3$};
    \node[main node] (E) at (2.75, 0) {$E$};

    \draw[->] (E) edge (X1);
    \draw[->] (Y) edge (X2);
    \draw[->] (X1) edge (X2);
    \draw[->] (X3) edge (Y);
    \draw[<->] (Y) edge (X1);
  \end{tikzpicture}
  }
  \qquad
  \subfloat[]{
    \begin{tikzpicture}[>=stealth',shorten >=1pt,auto,node distance=2cm,main node/.style={minimum size=0.6cm,font=\sffamily\Large\bfseries},scale=0.7,transform shape]

    \node[main node] (Y) at (0,0) {$Y$};
    \node[main node] (X1) at (1.5, 0) {$X_1$};
    \node[main node] (X2) at (0.75, -1.125) {$X_2$};
    \node[main node] (X3) at (-1.5, 0) {$X_3$};
    \node[main node] (E) at (3, 0) {$E$};

    \draw[o->] (E) edge (X1);
    \draw[->] (Y) edge (X2);
    \draw[->] (X1) edge (X2);
    \draw[o->] (X3) edge (Y);
    \draw[<->] (Y) edge (X1);
  \end{tikzpicture}
  }
    \caption{(a) An example ADMG. (b) PAG representing invariance spec structure for (a).}
  \label{fig:subgraph}
  \vspace{-0.2in}
\end{figure}

\textbf{2) Declaring Mutable Variables $\mathbf{M}$:}
Given the graph $\mathcal{P}$, to complete the invariance spec we must declare the mutable variables $\mathbf{M}$ (Line 2). The graph suggests possible mechanism shifts that occurred across the datasets: the \emph{possible children} of the environment indicator ($PossCh(E)$; nodes adjacent to $E$ with the edge not into $E$). In Fig \ref{fig:subgraph}b, $ PossCh(E) = \{X_1\}$. When there are many possible children of $E$, an advantage of having an explicit graph is that we can reason about and protect against only the shifts that are most likely to be problematic (vs defaulting to $\mathbf{M}=PossCh(E)$). We demonstrate this process on a mortality prediction problem in our experiments.

\textbf{3) Determining a Stable Distribution:}
Once we have the invariance spec $\langle \mathcal{P}, \mathbf{M} \rangle$, we need to find an identifiable\footnote{Recall that identifiability means that an interventional distribution is a function of the observational training data distribution.} interventional distribution that is stable to mechanism shifts in $\mathbf{M}$. In particular, we want to select the one that produces a model which performs best on heldout validation data. 
We established in Proposition \ref{prop:do-stability} that conditional interventional distributions which intervene on $\mathbf{M}$ are stable (i.e., distributions of the form $P_\mathbf{M}(Y|\mathbf{Z}\setminus\mathbf{M}),\mathbf{Z}\subseteq \mathbf{O}\setminus\{Y\}$). To check identifiability, we use two existing graphical criteria in PAGs (Lines 5,7), but delay further discussion of these until the next section (\ref{subsec:connections}).\footnote{CIDP and \citet[Thm 30]{zhang2008causal} are given in Appendix \ref{app:CIDP}.}

Unfortunately, searching for the optimal conditioning set that yields a stable identifiable interventional distribution (Alg \ref{alg}, Lines 4-8), like feature subset selection, is NP-Hard. Following related methods (e.g., \citet{magliacane2018domain,subbaswamy2019preventing}) we consider an exhaustive search over the feature powerset $\mathscr{P}(\mathbf{O}\setminus \{Y\})$, but note that many strategies for improving scalability exist, including greedy searches \citep{rojas2018invariant} or space pruning (using, e.g., $L_1$ regularization).

\textbf{Correctness:} We can now establish that Algorithm \ref{alg} does, in fact, return distributions which are guaranteed to be stable to the specified shifts.

\begin{corollary}[Soundness]\label{cor:sound}
If Algorithm \ref{alg} returns a distribution $P(Y|\mathbf{Z}, do(\mathbf{M}))$, then this is stable to shifts in $\mathbf{M}$.
\end{corollary}

\subsection{Connection to Dataset-driven Methods}\label{subsec:connections}

We now show that \textsc{I-Spec} subsumes the ability of existing dataset-driven approaches to find stable (conditional) distributions.\footnote{Namely, \citet{rojas2018invariant,magliacane2018domain} are easily adaptable to the setting of this paper; see Appendix \ref{app:relate}.} This is a consequence of the fact that stable conditional distributions are stable interventional interventional distributions as discussed next.

A prior result provides a sound and complete criterion for cases in which interventional distributions in PAGs reduce to conditional distributions \citet[Theorem 30]{zhang2008causal}. We adapt the criterion (Line 5) to find stable conditional distributions: cases in which $P_\mathbf{M}(Y|\mathbf{Z}) = P(Y|\mathbf{W}), \mathbf{W}\subseteq\mathbf{O}$. Distributions satisfying Line 5 are exactly the stable distributions that can found by existing data-driven methods.

However, not all identifiable interventional distributions reduce to conditionals, and are instead functionals of the observational distribution. These can be found using the CIDP algorithm \citep{jaber2019idc}.\footnote{CIDP has not been proven complete.} For example, in Fig \ref{fig:subgraph}b, if we consider the spec $\langle \mathcal{P}, \mathbf{M}=\{X_1\} \rangle$, then $P_{X_1}(Y|X_3, X_2)\propto P(Y|X_3)P(X_2|Y, X_1)$ via CIDP, while the only stable conditional distribution that can be found via Line 5 is $P_{X_1}(Y|X_3)=P(Y|X_3)$. We can now prove the main result of this section:

\begin{lemma}\label{lemma:cond}
Suppose a dataset-driven method finds $P(Y|\mathbf{Z}),\mathbf{Z}\subseteq\mathbf{O}$ to be stable given the input to Algorithm \ref{alg}. Then Algorithm \ref{alg} finds this distribution to be stable as well.
\end{lemma}

\begin{lemma}\label{lemma:do-see}
Algorithm \ref{alg} finds stable distributions that cannot be expressed as conditional observational distributions.
\end{lemma}

The following is now immediate:
\begin{corollary}\label{cor:subsume}
Algorithm \ref{alg} subsumes methods that find stable conditional (observational) distributions.
\end{corollary}

\subsection{Special Case: Only One Source Dataset}\label{subsec:one}
\textsc{I-Spec} was constructed to take datasets from multiple environments as input to match the input of existing dataset-driven methods that, by default, require this. We briefly want to note that \textsc{I-Spec} is easily extensible to the case in which only data from a single environment is available. In this case, there is no environment indicator and one can simply learn a PAG $\mathcal{P}$ over $\mathbf{O}$. Now specification of the mutable variables must come from prior knowledge alone, but we note that this is how selection variables are typically placed \citep{pearl2011transportability}. This yields an invariance spec $\langle \mathcal{P}, \mathbf{M} \rangle$ and stable interventional distributions can be found as before (i.e., Lines 4-8 of Alg \ref{alg}). While it may be possible to modify other methods to only require one dataset, we believe the extension to this setting is most natural using \textsc{I-Spec} because it uses an explicit graph.

\section{Related Work}\label{sec:rel-work}
\textbf{Proactively Addressing Dataset Shift:}
The problem of differing train and test distributions is known as \emph{dataset shift} \citep{quionero2009dataset}. Typical solutions assume access to unlabeled samples from the test distribution which are used to reweight training data during learning (e.g., \citet{shimodaira2000improving,gretton2009covariate}). However, in many practical applications it is infeasible to acquire test distribution samples. In this paper we consider shifts of arbitrary strengths when test samples are not available during learning, though there has been other work on \emph{bounded magnitude distributional robustness} when shifts are of a known type and strength \citep{rothenhausler2018anchor,heinze2017conditional}.

Dataset-driven approaches use datasets from multiple training environments to determine a feature subset \citep{rojas2018invariant,magliacane2018domain,kuang2018stable} or feature representation (e.g., \citet{muandet2013domain,arjovsky2019invariant}) that yields a conditional distribution that is invariant across the training datasets. Perhaps most related is \citet{magliacane2018domain}, whose method uses unlabeled target environment data, though it can be easily adapted to the setting of this paper. Notably, they allow for multiple environment (or ``context'') variables, and additionally consider shifts in environment due to a variety of types of interventions. Dataset-driven methods do not require an explicit causal graph, and by default conservatively protect against all shifts they detect across datasets.

In contrast, some works assume explicit knowledge of the underlying graph (i.e., an ADMG) so that users can specify the shifts in mechanisms to protect against. \citet{subbaswamy2019preventing} determine stable \emph{interventional} distributions in \emph{selection diagrams} \citep{pearl2011transportability} that can be used for prediction. Under the assumption of linear mechanisms, \citet{subbaswamy2018counterfactual} find a stable feature set that includes \emph{counterfactual} features. When there are no unobserved confounders, \citet{schulam2017reliable} protect against shifts in action policies and consider continuous-time longitudinal settings. \textsc{I-Spec} allows for unobserved variables and inherits the benefits of using interventional distributions, but relaxes the need for a fully specified graph, instead using a partial graph learned from data.

\textbf{Causal Discovery Across Multiple Environments:}
One line of research has focused exclusively on the problem of learning causal graphs using data from multiple environments. These methods could help extend \textsc{I-Spec} to other settings: For example, methods have been developed to learn a causal graph using data collected from multiple experimental contexts \citep{mooij2016joint,he2016causal} or non-stationary environments \citep{zhang2017causal}. The FCI variant described in Section \ref{subsec:framework} might be viewed as a special case of FCI-JCI \citep{mooij2016joint}, which allows for multiple environment/context variables.
\citet{triantafillou2010learning} consider the problem of learning a joint graph using datasets with different, but overlapping, variable sets. Others have considered local problems, e.g., using invariant prediction to infer a variable's causal parents \citep{peters2016causal,heinze2018invariant} or Markov blanket \citep{yu2019learning} when there are no unobserved confounders.

\section{Experiments}

We perform two experiments to demonstrate the efficacy of \textsc{I-Spec}. In our first experiment, we show that we can apply the framework to large, complicated datasets from the healthcare domain. Specifically, we are able to learn a partial graph that provides meaningful insights into both how the variables are related and what shifts occurred across datasets. We show how these insights can inform the choice of invariance spec, further showcasing the flexibility of the procedure since we can consider different choices of the mutable variables. We empirically show that \textsc{I-Spec} finds distributions that generalize well to new environments and produce consistent predictions irrespective of the choice of training environment. In our second experiment, we measure the degree to which the magnitude of shifts in environments affects the difference in performance between stable and unstable models. We used simulated data to create a large number of datasets in order to compare performance under varying shifts. These results confirm that stable models have more consistent performance across shifted environments and that interventionals can capture more stable information than conditionals.

\subsection{Real Data: Mortality Prediction}
\textbf{Motivation and Dataset:}
Machine learning has been used to predict intensive care unit (ICU) mortality to perform patient triage and identify most at-risk patients (e.g., \citet{pirracchio2015mortality}). However, in addition to physiologic features, studies have shown that features related to \emph{clinical practice patterns} (e.g., ordering frequency of lab tests) are highly predictive of patient outcomes \citep{agniel2018biases}. Since these patterns vary greatly by hospital, accurate models trained at one hospital will have highly variant performance at others, which can lead to unreliable and potentially dangerous decisions when deployed \citep{schulam2017reliable}. Therefore, we apply the proposed method to learn an ICU mortality prediction model that is stable to shifts in the mechanisms of such practice-based features and will generalize well to new hospitals. We demonstrate this using data from ICU patients at a large hospital and test its ability to generalize to smaller hospitals. 

We extract the first 24 hours of ICU patient data from three hospitals in our institution's network over a two year period.\footnote{Full inclusion criteria and details in Appendix \ref{app:exp}.} The pooled dataset consists of 24,787 individuals: 16,608 from Hospital 1 (H1); 5,621 from Hospital 2 (H2); and 2,558 from Hospital 3 (H3). We also extract 17 features, including the worst value of 12 physiological variables (e.g., heart rate), age, type of admission (i.e., surgical or medical), and three underlying chronic diseases (e.g., metastatic cancer), which are the features used in the SAPS II score \citep{le1993new}. To explicitly create a problematic shift, we simulate one practice-based variable: time of day when lab measurements occur (i.e., morning or night), whose correlation with mortality varies by hospital: mortality is correlated with morning measurements at H1, uncorrelated with measurement timing at H2, and correlated with night measurements at H3.

\textbf{Determining the Invariance Spec:}
To determine the invariance spec $\langle \mathcal{P}, \mathbf{M} \rangle$, we first learned a PAG $\mathcal{P}$ from the full pooled dataset, using the hospital ID as the environment indicator $E$. Specifically, using Tetrad we applied the FCI variant described in Section \ref{subsec:framework} and used the the Degenerate Gaussian likelihood ratio test \citep{andrews2019learning}. The learned PAG is given in Appendix \ref{app:pag}; we describe some aspects of it here to demonstrate its value.

12 variables are possible children of $E$, including physiologic variables such as `Age' and `Bicarbonate', and features associated with clinical practice such as `Admit Type' and Lab Time'. Of the 10 variables adjacent to `Mortality', `Age' is the only parent---the other 9 variables are connected via bidirected edges. The explicit graph makes it easy to reason about the DGP: it tells us that `Age' is a causal factor for mortality (e.g., older patients are more likely to die), while `Bicarbonate' is related to mortality through unobserved common causes (such as an acute underlying kidney condition). The bidirected edge connecting `Lab Time' to `Mortality' indicates a practice-based non-causal relationship, and the bidirected edge between `Admit Type' and `Mortality' is due to the latent condition that caused the admission and will contribute to risk of mortality.

In this example, if a model is not stable to shifts in practice pattern-based features, then the predictions it makes will be arbitrarily sensitive to changes in policies between datasets, such as shifts in the times when lab measurements are taken. This sensitivity would render the model \emph{unreliable}, so we reason that shifts in administrative policies should not affect our mortality risk predictions. In contrast, shifts in physiologic mechanisms may encode clinically relevant changes: if there are differences in the treatments patients receive at different hospitals, this would affect the bicarbonate mechanism, for example. Because these shifts would be clinically meaningful, they should affect the decisions we make and model predictions should not be invariant to them. Thus, one reasonable invariance spec is to take the mutable variables to be $\mathbf{M}=\{\textnormal{`Admit Type', `Lab Time'}\}$. Note the flexibility of this procedure: we are able to consider alternative invariance specs (i.e., different choices of $\mathbf{M}$) and compare the sensitivity of resulting solutions.

\textbf{Baselines/Models:}
We consider three models that correspond to the three ways a model developer can respond to shifts in environment: ignoring shifts, protecting against all shifts in the datasets, or protecting against some shifts. Our first baseline, an \textbf{unstable} model, ignores shifts and uses all features. Our second baseline conservatively protects against all shifts in the data by using \textsc{I-Spec} with the invariance spec $\langle \mathcal{P}, PossCh(E) \rangle$, emulating the \textbf{conservative} default data-driven behavior. Finally, to protect against only some shifts we use \textbf{I-SPEC} and the invariance spec $\langle \mathcal{P}, \mathbf{M} \rangle$ defined before.  Using the procedure in Alg \ref{alg}, \textsc{I-Spec} uses 13 of the 18 features, while the conservative method uses 7---neither use `Lab Time' or `Admit Type'. For demonstration we train logistic regression models, though we emphasize that more complex models could be used instead.

\begin{figure}[!t]
\centering
\includegraphics[width=0.75\textwidth]{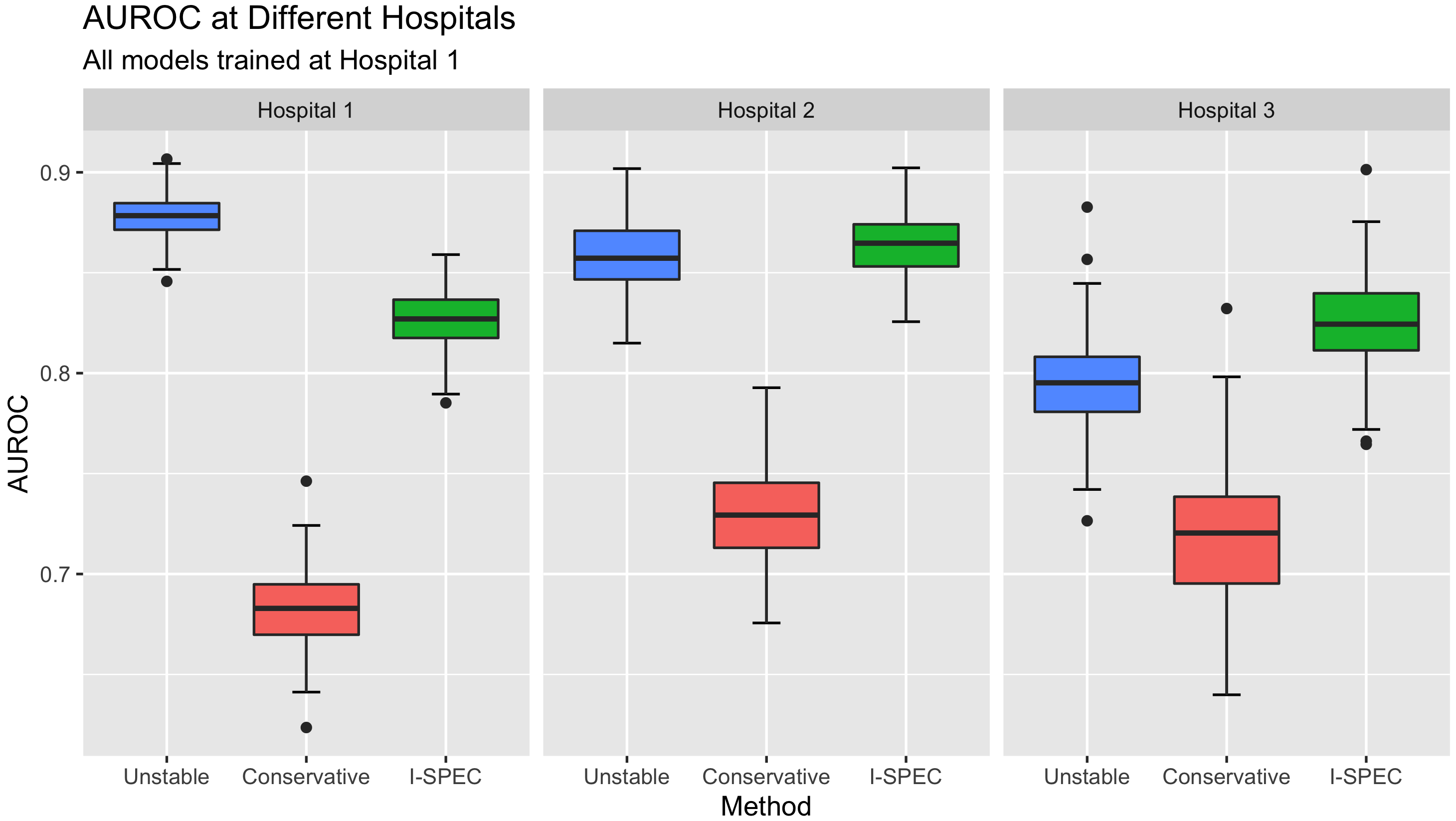}
\caption{Performance (AUROC) of unstable (blue; medians 0.88, 0.86, 0.80), stable conservative (red; medians 0.69, 0.73, 0.72), and stable \textsc{I-Spec} (green; medians 0.83, 0.87, 0.82) mortality prediction models trained at Hospital 1 but evaluated at each hospital.}
\label{fig:auroc}
\end{figure}

\textbf{Experimental Setup:}
We evaluate as follows: We randomly performed 80/20 train/test splits on data from each hospital (and repeated this 100 times). To measure predictive performance, we use the H1 dataset to train the unstable, conservative, and \textsc{I-Spec} models, and evaluated their area under the ROC curve (AUROC) on the test patients from each hospital. This allows us to see the robustness of a model's performance as it is applied to new environments. Beyond performance, we also evaluated the effect of shifts on model decisions. For each approach, we consider pairs of models (one trained at H1, and one trained at H2 or H3) and made predictions on the test set patients. We then computed the rank correlation of the predictions via Spearman's $\rho$. A value of $\rho=1$ indicates that two models produce the same ordering of patients by predicted risk despite being trained at different hospitals (i.e., patient orderings are stable).

\textbf{Results:}
Fig \ref{fig:auroc} shows boxplots of the AUROC of the models at each test hospital.
As expected, the unstable model fails to generalize to new hospitals, with a significant drop in performance from H1 to H3 because the unstable lab time-mortality association flipped. On the other hand, the \textsc{I-Spec} model generalizes well, and outperforms the unstable model at the new hospitals H2 and H3. Comparing \textsc{I-Spec} to the conservative model, we see that the conservative model performs worse at all hospitals precisely because it protects against all shifts (leaving less predictive signal to learn), though its performance also does not deteriorate at new hospitals because the model is also stable.

\begin{figure}[!th]
\centering
\includegraphics[width=0.75\textwidth]{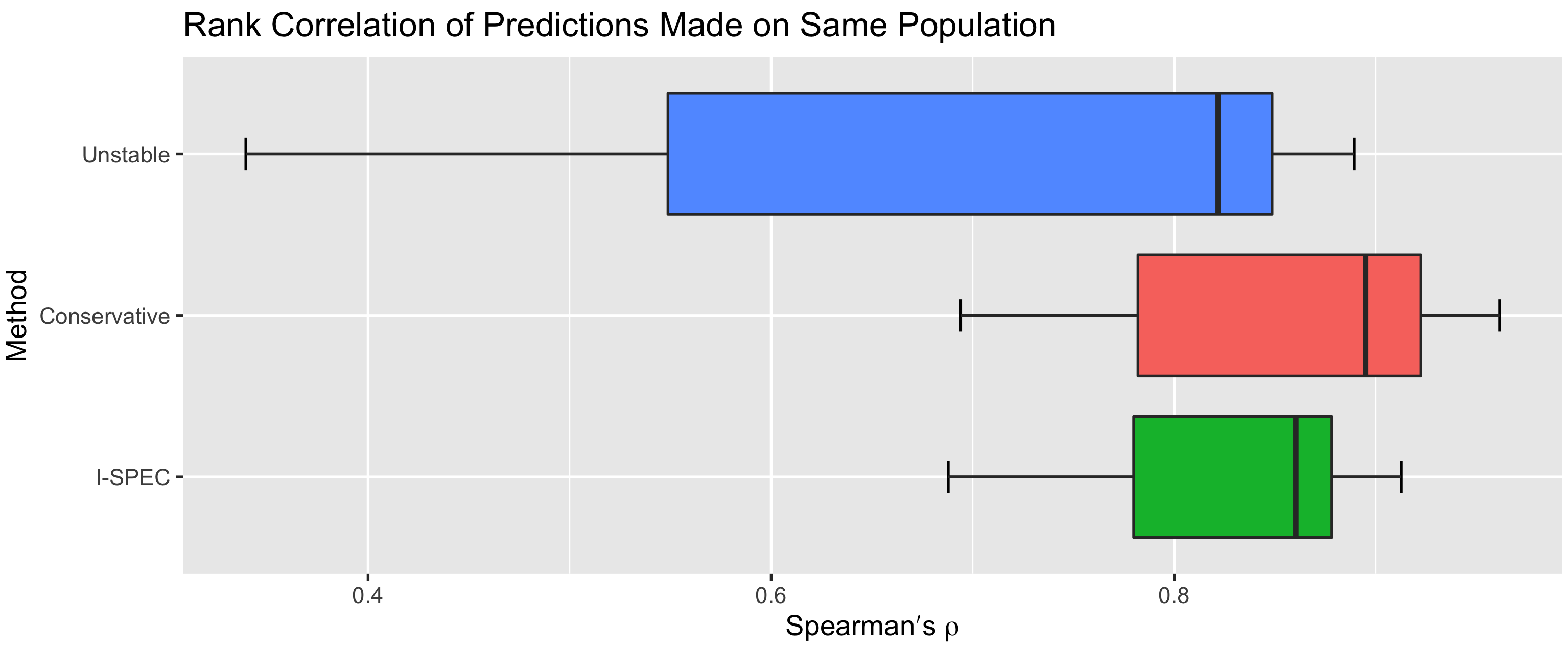}
\caption{Rank correlation between predictions by models trained at different hospitals but applied to the same test patients. Median $\rho$'s: unstable (0.82), conservative (0.90), \textsc{I-Spec} (0.86).}
\label{fig:rank}
\end{figure}

Fig \ref{fig:rank} shows the boxplots of rank correlations of each model's predictions. The unstable model has significantly less stable patient orderings than the two stable models: its rank correlations are highly varying and reach as low as $\rho = 0.34$. Both the \textsc{I-Spec} and conservative models have similar rank correlations, though the conservative model's $\rho$'s tend to be slightly higher due to protecting against all shifts. Overall, we see that stable models produce significantly more consistent predictions (and, thus, more stable patient orderings) than the unstable model. The difference between the stable models, however, is that the \textsc{I-Spec} model has significantly and strictly better discriminative performance at all hospitals. This demonstrates that careful choice of the mutable variables (as opposed  to defaulting to $\mathbf{M}=PossCh(E)$), can yield stable and accurate models.

\subsection{Simulated Data}

\begin{figure}[!th]
\centering
\includegraphics[width=0.65\textwidth]{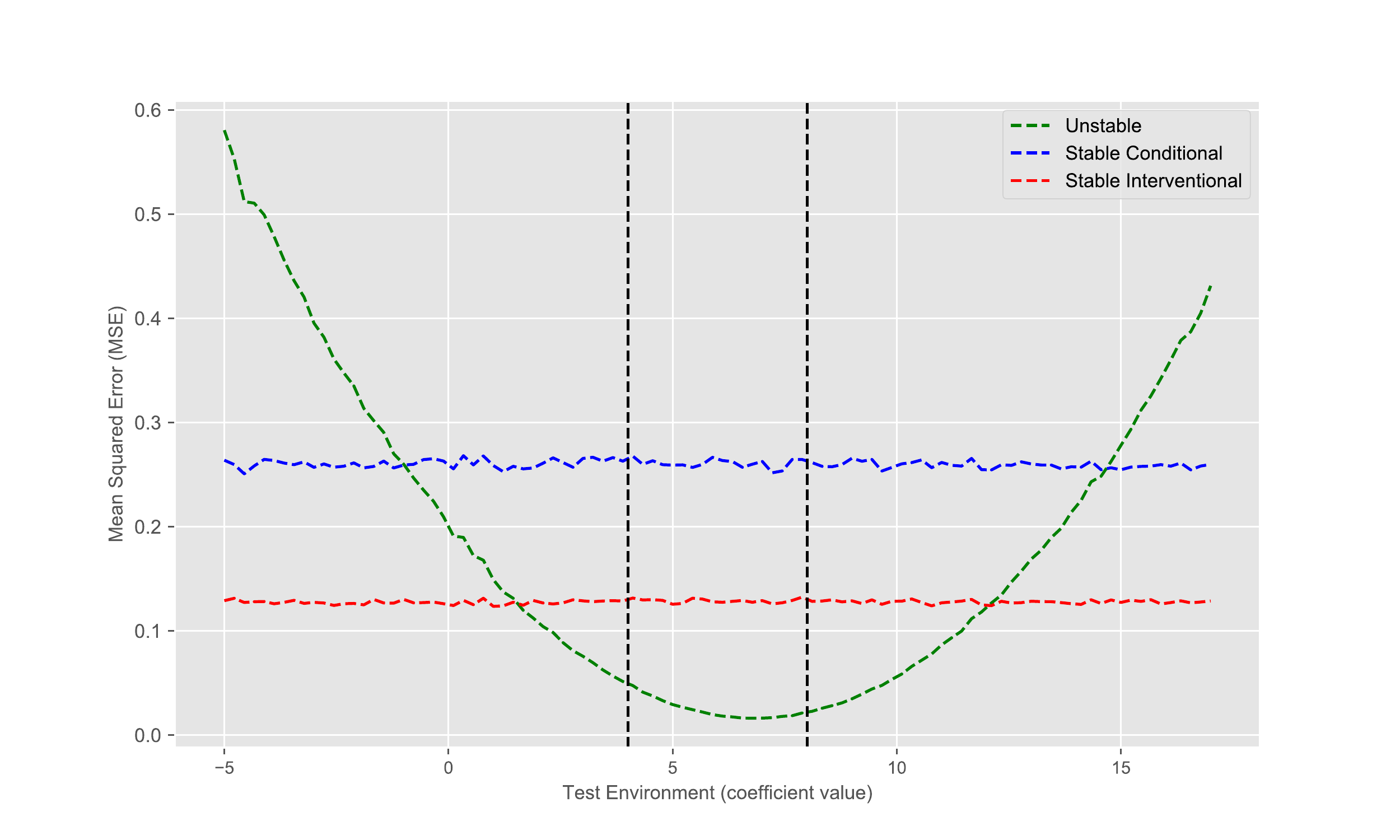}
\caption{MSE of different models as they are evaluated in different test environments. Vertical dashed lines denote the coefficient values associated with the two training environments.}
\label{fig:sim}
\end{figure}

To analyze the effect of the magnitude of shifts on the performance of stable and unstable models, we simulated data from a zero-mean linear Gaussian system according to the ADMG in Fig \ref{fig:subgraph}a. We shift the mechanism of $X_1$ by changing the coefficient of the unobserved confounder between $Y$ and $X_1$ in the structural equation for $X_1$.\footnote{Exact simulation details in Appendix \ref{app:exp}.} We generated two source datasets (environments denoted by the vertical dashed lines in Fig \ref{fig:sim}) and trained three linear regression models: an unstable (green) model of $E[Y|X_1,X_2,X_3]$, a stable conditional (blue) model of $E[Y|X_3]$, and a stable interventional (red) model of $E[Y|do(X_1), X_2, X_3]$. We then evaluated the mean squared error (MSE) of these models (plotted in Fig \ref{fig:sim}) in test environments created by varying the unstable coefficient.

As expected, under small shifts the unstable model outerperforms both stable models; but the unstable model's error grows rapidly with the magnitude of the shift quickly performing much worse than the stable models. On the other hand, the stable models have performance that is consistent across environments as desired. The interventional model achieves lower MSE than the conditional because it uses stable information in the $P(X_2|Y,X_1)$ term that the conditional model does not.

\section{Conclusion}

In this paper we addressed one of the primary challenges facing the deployment of machine learning in safety-critical applications: shifts in environment between training and deployment. To this end, we proposed \textsc{I-Spec}, an end-to-end framework that lets us go from data to models that are guaranteed to be stable to shifts. Like existing graphical methods, \textsc{I-Spec} does not require data from the target environment and is able to capture more stable information in the data than methods which use stable conditional distributions. An important difference, however, is that \textsc{I-Spec} does not require prior knowledge of the full causal graph. As demonstrated in our healthcare experiments, this means \textsc{I-Spec} can be applied to problems in which existing graphical methods would have been too difficult to use. The experiments further demonstrated how the framework can be used to discover shifts, determine which ones to protect against, and train accurate, stable models. To improve \textsc{I-Spec}'s interoperability, a valuable direction for future work would be to handle differing variables sets across datasets. 
\section*{Acknowledgements}
The authors thank Dan Malinsky for helpful discussions about structure learning, the Tetrad developers for promptly providing an implementation of the Degenerate Gaussian score, and Sieu Tran for help in implementation of an earlier version of this work.

\bibliographystyle{apalike}
\bibliography{references}

\onecolumn
\newpage
\appendix
\section{Invariant Conditionals in PAGs and the CIDP Algorithm}\label{app:CIDP}
\subsection{Additional PAG Preliminaries}
We first provide some additional definitions and facts about PAGs. These are relevant for understanding Theorem \ref{thm:invar}.

The $d$-separation criterion in DAGs is naturally generalized to encode conditional independences in mixed graphs through \emph{$m$-separation} \citep{richardson2002ancestral}. A path is \emph{$m$-connecting} given a set $\mathbf{Z}$ if every collider (e.g., v-structure like $\rightarrow V \leftarrow $) on the path is in $an(\mathbf{Z})$ and all non-colliders are not in $\mathbf{Z}$.

In PAGs we must also account for uncertainty in whether or not a node is a collider along a path. Letting $*$ denote a wildcard edge mark (head, tail, or circle), a node $V_j$ is a \emph{definite non-collider} if there is at least one edge out of $V_j$ on the path, or if $V_i\starcirc V_j \circstar V_k$ is a subpath and $V_i$ and $V_k$ are not adjacent. A \emph{definite status} path is one in which every node is either a collider or definite non-collider \citep{maathuis2015generalized}. These definitions let us extend $m$-connection (and separation) to PAGs: a definite status path is \emph{m-connecting} given $\mathbf{Z}$ if every definite non-collider is in $\mathbf{Z}$ and every collider on the path is in $an(\mathbf{Z})$.

\subsection{Invariance Criterion for Conditionals}
\begin{theorem}[\citet{zhang2008causal}, Theorem 30] \label{thm:invar}
Suppose $\mathcal{P}$ is the PAG over the observed variables $\mathbf{O}$. For any $\mathbf{X,Y,Z} \subseteq \mathbf{O}$ such that $\mathbf{X} \intersection \mathbf{Y} = \mathbf{Y} \intersection \mathbf{Z} = \emptyset$, $P(\mathbf{y}|\mathbf{z})$ is invariant under interventions on $\mathbf{X}$ in $\mathcal{P}$ if and only if

\vspace{-0.5\baselineskip}
\begin{description}[align=left,leftmargin=*,labelindent=0in,topsep=0ex,itemsep=0ex,partopsep=0ex,parsep=0ex]
\item[1)] for every $X \in \mathbf{X}\intersection \mathbf{Z}$, every definite status m-connecting path, if any, between $X$ and any member of $\mathbf{Y}$ given $\mathbf{Z}\setminus \{X\}$ is out of $X$ with a visible edge;
\item[2)] for every $X \in \mathbf{X}\intersection (PossibleAn(\mathbf{Z})\setminus \mathbf{Z})$, there is no definite status m-connecting path between $X$ and any member of $\mathbf{Y}$ given $\mathbf{Z}$;
\item[3)] for every $X \in \mathbf{X}\setminus PossibleAn(\mathbf{Z})$, every definite status m-connecting path, if any, between $X$ and any member of $\mathbf{Y}$ given $\mathbf{Z}\setminus \{X\}$ is into $X$.
\end{description}
\end{theorem}

As originally written, verifying Theorem \ref{thm:invar} involves checking individual definite status paths in the PAG. We will reduce the conditions to equivalent ones that can be verified in MAGs derived from the PAG that will, in general, have fewer paths, and for which efficient m-separation routines have been implemented (e.g., in the \texttt{R} package \texttt{dagitty} \citep{textor2016robust}). First, we require the following definitions from \citet{maathuis2015generalized}, with the addition of $\mathcal{R}_{\overline{X}}$.

\begin{definition}[$\mathcal{R}^*$, $\mathcal{R}_{\underline{X}}$, and $\mathcal{R}_{\overline{X}}$]
Let $X$ be a vertex in PAG $\mathcal{P}$. Define $\mathcal{R}^*$ to be the set of MAGs in the equivalence class described by $\mathcal{P}$ that have the same number of edges into $X$ as in $\mathcal{P}$. For any $\mathcal{R}\in\mathcal{R}^*$, let $\mathcal{R}_{\underline{X}}$ be the graph obtained from $\mathcal{R}$ by removing all directed edges out of $X$ that are visible in $\mathcal{P}$. For any $\mathcal{R}\in\mathcal{R}^*$, let $\mathcal{R}_{\overline{X}}$ be the graph obtained from $\mathcal{R}$ by removing all edges (directed or bidirected) into $X$.
\end{definition}

Theorem \ref{thm:invar} can now be verified via Lemma \ref{lemma:mag}:

\begin{lemma}\label{lemma:mag}
For $\mathbf{X}$, $\mathbf{Y}$, $\mathbf{Z}$ as in Theorem \ref{thm:invar} and $\mathcal{R}\in\mathcal{R}^*$, the Theorem \ref{thm:invar} conditions are equivalent to
\vspace{-0.5\baselineskip}
\begin{description}[align=left,leftmargin=*,labelindent=0in,topsep=0ex,itemsep=0ex,partopsep=0ex,parsep=0ex]
    \item[1)] for every $X \in \mathbf{X}\intersection \mathbf{Z}$, $X \ci \mathbf{Y} | \mathbf{Z}\setminus \{X\}$ in $\mathcal{R}_{\underline{X}}$;
    \item[2)] for every $X \in \mathbf{X}\intersection (PossibleAn(\mathbf{Z})\setminus \mathbf{Z})$, $X \ci \mathbf{Y}|\mathbf{Z}$ in  $\mathcal{R}\in \mathcal{R}^* \subseteq \{ \mathcal{P} \}$ ;
    \item[3)] for $X \in \mathbf{X}\setminus PossibleAn(\mathbf{Z})$,  $X \ci \mathbf{Y} | \mathbf{Z}$ in $\mathcal{R}_{\overline{X}}$.
\end{description}
\end{lemma}%

\begin{proof}[Proof of Lemma \ref{lemma:mag}]
Consider each condition in turn.
\begin{enumerate}
    \item This equivalence is a restatement of Lemma 7.4 in \citet{maathuis2015generalized} (which states the condition as there is no m-connecting path between $X$ and $Y$ given $\mathbf{Z}\setminus\{X\}$ in $\mathcal{R}_{\underline{X}}$, i.e., they are m-separated).
    
    \item This equivalence follows from the definition of a PAG. All MAGs in the equivalence class represented by $\mathcal{P}$ share the same conditional independences. Thus, if $X \ci \mathbf{Y} | \mathbf{Z}$ in one MAG in $\{\mathcal{P}\}$ then $X \ci \mathbf{Y} | \mathbf{Z}$ in $\mathcal{P}$. Similarly, if $X \ci \mathbf{Y} | \mathbf{Z}$ in $\mathcal{P}$ then $X \ci \mathbf{Y} | \mathbf{Z}$ in all MAGs in  $\{\mathcal{P}\}$ .
    
    \item To prove this equivalence we must prove the following: Let $X \in \mathbf{X}\setminus PossibleAn(\mathbf{Z})$. Then there is a definite status m-connecting path from $X$ to $Y\in\mathbf{Y}$ given in $\mathcal{P}$ that is not into $X$ if and only if there is an m-connecting path between X and $Y$ given $\mathbf{Z}$ in $\mathcal{R}_{\overline{X}}$. The style of the proof follows that of the proof of Lemma 7.4 in \citet{maathuis2015generalized}.
    
    First, the only if direction. Suppose there is definite status m-connecting path, $p$, between $X$ and $Y$ given $\mathbf{Z}$ in $\mathcal{P}$ that is not into $X$. Let $p'$ be this path in $\mathcal{R}$ and $p''$ be this path in $\mathcal{R}_{\overline{X}}$. As noted in \citet{zhang2008causal}, if a path is definite status m-connecting, then the corresponding path in every MAG in $\mathcal{P}$ is m-connecting. Thus, we know that $p'$ is m-connecting. Further, $p'$ is out of $X$ because $p$ was not into $X$, and by construction $\mathcal{R}$ has no additional edges into $X$ when compared to $\mathcal{P}$. Since $\mathcal{R}_{\overline{X}}$ only deletes edges into $X$ when compared to $\mathcal{R}$, the path $p''$ is no different from $p'$ and is also out of $X$. $p''$ is also m-connecting because the only way for $p''$ to not be m-connecting while $p'$ is, would be for $p'$ to contain a collider that became inactive after deleting edges into $X$. However, we know that $p'$ (and, thus, $p''$) are not collider paths since $X \not \in An(\mathbf{Z})$, and thus the paths $p'$ and $p''$ are directed and out of $X$. Thus, $p''$ is m-connecting and out of $X$ in $\mathcal{R}_{\overline{X}}$.
    
    Now the if direction. Suppose there is an m-connecting path $p''$ between $X$ and $Y$ given $\mathbf{Z}$ in $\mathcal{R}_{\overline{X}}$. Because this path is out of $X$, the corresponding path $p'$ in $\mathcal{R}$ is unaffected and is also m-connecting. By Lemma 5.1.9 in \citet{zhang2006causal}, since $\mathcal{R}\in\{\mathcal{P}\}$, this means there is a definite status m-connecting path, $p$, between $X$ and $Y$ given $\mathbf{Z}$ in $\mathcal{P}$ that is not into $X$.
\end{enumerate}
\end{proof}

\subsection{CIDP Algorithm}
We briefly restate key aspects of the CIDP algorithm here. For full details see \citet{jaber2019idc}.

\citet{jaber2019idc} introduce additional constructs that are used in the CIDP algorithm. In what follows, we will use $Pa^+(\mathbf{X})$ ($Ch^+(\mathbf{X})$) to denote the union of $\mathbf{X}$ and the set of possible parents (children). Similarly, we define $An^+(\mathbf{X})$. We will let $Pa^*(\mathbf{X})$ denote $Pa^+(\mathbf{X})$ excluding the possible parents of $\mathbf{X}$ due to circle edges. We similarly define $Ch^*(\mathbf{X})$. Let a circle path be a path on which all edge marks are $\circ$. Define a \emph{bucket} to be a closure of nodes connected with circle paths as a bucket.

\begin{definition}[PC-Component]
In a PAG or any induced subgraph thereof, two nodes are in the same possible c-component (pc-component) if there is a path between them such that (1) all non-endpoint nodes along the path are colliders, and (2) none of the edges are visible.
\end{definition}

Note that two nodes are in the same \emph{definite c-component} if they are connected by a bi-directed path.

The following proposition gives an identification criterion for interventional distributions corresonding to interventions on a bucket.
\begin{proposition}[\cite{jaber2019idc}[Proposition 2]]\label{prop:bucket}
Let $\mathcal{P}$ denote a PAG over $\mathbf{V}$, $\mathbf{T}$ be a union of a subset of buckets in $\mathcal{P}$, and $\mathbf{X}\subset \mathbf{T}$ be a bucket. Given $P_{\mathbf{V}\setminus\mathbf{T}}$ (i.e., $Q[T]$), and a partial topological order of buckets $\mathbf{B}_1 < \dots < \mathbf{B}_m$ with respect to $\mathcal{P}_\mathbf{T}$ (induced subgraph), $Q[\mathbf{T}\setminus\mathbf{X}]$ is identifiable if and only if, in $\mathcal{P}_\mathbf{T}$, there does not exist $Z\in\mathbf{X}$ such that Z has a possible child $C\not\in\mathbf{X}$ that is in the pc-component of Z. If identifiable, then the expression is given by
$$Q[\mathbf{T}\setminus\mathbf{X}] = \frac{P_{\mathbf{V}\setminus\mathbf{T}}}{\prod_{i|\mathbf{B}_i\subseteq S^\mathbf{X}}P_{\mathbf{V}\setminus\mathbf{T}}(\mathbf{B}_i|\mathbf{B}^{(i-1)})}\times \sum_\mathbf{X} \prod_{i|\mathbf{B}_i\subseteq S^\mathbf{X}} P_{\mathbf{V}\setminus\mathbf{T}}(\mathbf{B}_i|\mathbf{B}^{(i-1)},$$
where $S^\mathbf{X}$ is the union of the definite c-components of the members of $\mathbf{X}$ in $\mathcal{P}_\mathbf{T}$, and $\mathbf{B}^{(i-1)}$ denotes the set of nodes preceding bucket $\mathbf{B}_i$ in the partial order.
\end{proposition}

\begin{definition}[Region $\mathcal{R}_\mathbf{A}^\mathbf{C}$]
Given a PAG $\mathcal{P}$ over $\mathbf{V}$, and $\mathbf{A}\subseteq\mathbf{C}\subseteq\mathbf{V}$. Let the region of $\mathbf{A}$ w.r.t. $\mathbf{C}$, denoted $\mathcal{R}_\mathbf{A}^\mathbf{C}$, be the union of the buckets that contain nodes in the pc-component of $\mathbf{A}$ in the induced subgraph $\mathcal{P}_\mathbf{C}$.
\end{definition}

We are now ready to state the algorithm.

\begin{algorithm}[!h]
\small
 \LinesNumbered
 \SetKwFunction{FIdentify}{Identify}
  \SetKwFunction{FDecompose}{Decompose}
   \SetKwFunction{FDoSee}{Do-See}
 \SetKw{throw}{throw}

 \SetKwProg{Fn}{Function}{:}{}
 \SetKwInOut{Input}{input}\SetKwInOut{Output}{output}
 \Input{three disjoint sets $\mathbf{X},\mathbf{Y},\mathbf{Z}\subset\mathbf{V}$}
 \Output{Expression for $P_\mathbf{X}(\mathbf{Y}|\mathbf{Z})$ or \texttt{FAIL}.}
 Let $\mathbf{D}=An^+(\mathbf{Y}\cup\mathbf{Z})_{\mathcal{P}_{\mathbf{V}\setminus\mathbf{X}}\setminus\mathbf{Z}}$\;
 $P_\mathbf{X}(\mathbf{Y}|\mathbf{Z}) = \sum_{\mathbf{D}\setminus\mathbf{Y}} Q[\mathbf{D}|\mathbf{Z}]$\;
 $\mathbf{F}$ = \FDecompose{$\mathcal{P},\mathbf{D},\mathbf{Z}$}\;
 Let $\mathbf{F}^*=\emptyset$\;
 \For{$\langle \mathbf{D}_i, \mathbf{Z}_i \rangle \in \mathbf{F}$}{
    \If{$\mathbf{D}_i \intersection \mathbf{Y} \not = \emptyset$}{
        $\mathbf{F}^* = \mathbf{F}^* \cup$ \FDoSee{$\mathcal{P}, \mathbf{D}_i, \mathbf{Z_i}$}\;
    }
 }
 $P_\mathbf{X}(\mathbf{Y}|\mathbf{Z})=\prod_{i|\langle \mathbf{D}_i, \mathbf{Z}_i \rangle \in \mathbf{F}^*} \sum_{\mathbf{D}_i\setminus\mathbf{Y}}\frac{\textrm{\FIdentify{$\mathbf{D}_i\cup\mathbf{Z}_i, \mathbf{V}, P$}}}{\sum_{\mathbf{D}_i}\textrm{\FIdentify{$\mathbf{D}_i\cup\mathbf{Z}_i, \mathbf{V}, P$}} }$\;

 \Fn{\FDecompose{$\mathcal{P},\mathbf{T},\mathbf{Z}$}}{
 \If{$\mathbf{T}==\emptyset$}{\KwRet $\emptyset$\;}
 \tcc{In $\mathcal{P}_{\mathbf{T}\cup\mathbf{Z}}$ let $C^{(\cdot)}$ denote the pc-component of $(\cdot)$ in $\mathcal{P}_{\mathbf{T}\cup\mathbf{Z}}$}
 Initialize $\mathbf{X}$ to some node in $\mathbf{T}$\;
 Let $\mathbf{A}=Pa^*(C^\mathbf{X})\intersection Pa^*(C^{\mathbf{T}\cup\mathbf{Z}\setminus C^\mathbf{X}})$\;
 
 \While{$\mathbf{A}\not\subseteq\mathbf{Z}$}{
    $\mathbf{X} = \mathbf{X} \cup Ch^*(\mathbf{A}\intersection\mathbf{T})$\;
    $\mathbf{A} = Pa^*(C^\mathbf{X})\intersection Pa^*(C^{\mathbf{T}\cup\mathbf{Z}\setminus C^\mathbf{X}})$\;
 }
 \tcc{Let $\mathbf{T}_1=C^\mathbf{X}\intersection\mathbf{T}$ and $\mathbf{T}_2=\mathbf{T}\setminus\mathbf{T}_1$}
 \KwRet $\langle \mathbf{T}_1, \mathcal{R}_{\mathbf{X}}\setminus\mathbf{T}_1 \rangle \cup $\FDecompose{$\mathcal{P}, \mathbf{T}_2, \mathcal{R}_{\mathbf{T}\cup\mathbf{Z}\setminus C^\mathbf{X}}\setminus\mathbf{T}_2$}\;
 }
 \Fn{\FDoSee{$\mathcal{P},\mathbf{T},\mathbf{Z}$}}{
    \tcc{Let $\mathbf{B}$ denote a bucket in $\mathcal{P}$ and $C^{(\cdot)}$ denote the pc-component of $(\cdot)$ in $\mathcal{P}_{\mathbf{T}\cup\mathbf{Z}\cup\mathbf{B}}$}
    
    \If{$\exists \mathbf{B} | \mathbf{B} \intersection (\mathbf{T} \cup \mathbf{Z}) \not = \emptyset \land \mathbf{B}\not\subseteq(\mathbf{T}\cup\mathbf{Z})$}{
        \If{$Pa^*(C^{\mathbf{B}\setminus(\mathbf{T}\cup\mathbf{Z})}\intersection \mathbf{T} = \emptyset$}{
            \KwRet \FDoSee{$\mathcal{P}, \mathbf{T}, \mathbf{Z}\cup\mathbf{B}\setminus\mathbf{T}$}\;
        }
        \Else{
        \throw \texttt{FAIL}\;}
    }
    \KwRet $\langle \mathbf{T}, \mathbf{Z} \rangle$\;
 }
 
  \Fn{\FIdentify{$\mathbf{C},\mathbf{T},Q=Q[\mathbf{T}]$}}{
    \If{$\mathbf{C}=\emptyset$}{\KwRet 1}
    \If{$\mathbf{C} = \mathbf{T}$}{\KwRet Q}
    \tcc{In $\mathcal{P}_\mathbf{T}$ let $\mathbf{B}$ denote a bucket, and let $C^\mathbf{B}$ denote the pc-component of $\mathbf{B}$}
    
    \If{$\exists\mathbf{B}\subset\mathbf{T}\setminus\mathbf{C}$ such that $C^\mathbf{B}\intersection Ch^+(\mathbf{B})\subseteq\mathbf{B}$}{
        Compute $Q[\mathbf{T}\setminus\mathbf{B}]$ from $Q$ using Proposition \ref{prop:bucket}\;
        \KwRet \FIdentify{$\mathbf{C},\mathbf{T}\setminus\mathbf{B}, Q[\mathbf{T}\setminus\mathbf{B}]$}\;
    }
    \ElseIf{$\exists \mathbf{B}\subset\mathbf{C}$ such that $\mathcal{R}_\mathbf{B} \not = \mathbf{C}$}{
    \KwRet $\frac{\textrm{\FIdentify{$\mathcal{R}_\mathbf{B},\mathbf{T}, Q$}}\cdot
    \textrm{\FIdentify{$\mathcal{R}_{\mathbf{C}\setminus\mathcal{R}_\mathbf{B}},\mathbf{T}, Q$}}
    }{\textrm{\FIdentify{$\mathcal{R}_\mathbf{B} \intersection\mathcal{R}_{\mathbf{C}\setminus\mathcal{R}_\mathbf{B}},\mathbf{T}, Q$}}}$
    }
    \Else{
    \throw \texttt{FAIL}\;
    }
  }
 
 \caption{CIDP($\mathbf{X},\mathbf{Y},\mathbf{Z}$) given PAG $\mathcal{P}$}
 \label{alg:cidp}
\end{algorithm}

\newpage

\section{Proofs of Main Results}\label{app:proofs}

\begin{proof}[Proof of Proposition \ref{prop:admiss}]
Follows from $S$-admissibility \citep[Theorem 2]{pearl2011transportability} and the definition of a PAG (independences that hold in every member of the PAG's equivalence class must also hold in the PAG).
\end{proof}

\begin{proof}[Proof of Proposition \ref{prop:do-stability}]
In each ADMG $\mathcal{G}$ in the equivalence class $\{\mathcal{P}\}$, we have that $Y \ci \mathbf{S} | \mathbf{Z}$ in $\mathcal{G}_{\overline{\mathbf{M}}}$, the mutilated graph in which all edges into $\mathbf{M}$ have been deleted due to the $do$ operator in $P(Y|do(\mathbf{M}), \mathbf{Z})$. Now, by Rule 2 of $do$-calculus we have that $P(Y|do(\mathbf{M}), \mathbf{Z},\mathbf{S})$ = $P(Y|do(\mathbf{M}),\mathbf{Z})$ (again in each ADMG in $\{\mathcal{P}\}$), so $P(Y|do(\mathbf{M}),\mathbf{Z})$ is a stable distribution by Proposition \ref{prop:admiss}.
\end{proof}

\begin{proof}[Proof of Corollary \ref{cor:sound}]
First, note that given an invariance spec $\langle \mathcal{P}, \mathbf{M} \rangle$, Algorithm \ref{alg} searches over distributions of the form $P(Y|do(\mathbf{M}),\mathbf{Z})$. All of these are stable by Proposition \ref{prop:do-stability}. Now, for conditioning sets that satisfy Theorem \ref{thm:invar}, these are in fact stable due to the fact that the theorem is a sufficient graphical conditional for invariance (see the ``if'' direction of the proof in \citet{zhang2008causal}). For conditional interventional distributions found to be identifiable by CIDP, correctness follows from its soundness \citet[Theorem 1]{jaber2019idc}.
\end{proof}

\begin{proof}[Proof of Lemma \ref{lemma:cond}]
The conditioning set $\mathbf{Z}$ found by the dataset driven method will be checked in Line 5 of Algorithm \ref{alg} to see if it satisfies Theorem \ref{thm:invar}. Because Theorem \ref{thm:invar} is sound and complete in PAGs, it is satisfied by all stable conditioning sets. Thus, Algorithm \ref{alg} will find $P(Y|\mathbf{Z})$ to be stable and will append it to $Stable$.
\end{proof}

\begin{proof}[Proof of Lemma \ref{lemma:do-see}]
Consider Fig \ref{fig:subgraph}b in which $P_{X_1}(Y|X_3, X_2) = \frac{P(Y|X_3)P(X_2|Y,X_1)}{\sum_{Y'}P(Y'|X_3)P(X_2|Y',X_1)}$ is stable but is not reducible to a conditional distribution of the form $P(Y|\mathbf{Z}),\mathbf{Z}\subseteq\mathbf{O}$.
\end{proof}

\section{Clarifying Relation to Dataset-Driven Approaches}\label{app:relate}

We now discuss how existing dataset-driven methods, \citet{rojas2018invariant} and \citet{magliacane2018domain} in particular, can be adapted to address a problem defined by an invariance spec. Then, by virtue of the fact that these methods search for invariant conditionals, this means that these methods are subsumed by \textsc{I-SPEC}.

First, \citet{rojas2018invariant} is related to work on invariant prediction that finds stable distributions by hypothesis testing the stability of a distribution across source environments \citep{peters2016causal}. While these works do not assume faithfulness, under the faithfulness assumption (which is made by \textsc{I-Spec}), it has been shown that an invariant distribution $P(Y|\mathbf{X})$ corresponds to a feature set $\mathbf{X}$ such that $Y\ci E | \mathbf{X}$: the target variable is $d$-separated from the environment indicator given the features \citep[Appendix C]{peters2016causal}. Thus, \citet{rojas2018invariant} can naturally be applied to the input of \textsc{I-Spec} and it searches for stable conditional distributions as defined within the main paper.

\citet{magliacane2018domain}, by contrast, builds on the Joint Causal Inference (JCI) framework proposed in \citet{mooij2016joint}. The JCI framework considers a related setting to the environment indicator setup (that is used by \textsc{I-Spec} and invariant prediction works like \citet{peters2016causal,rojas2018invariant}) in which there are instead (possibly multiple) \emph{context} variables that describe how environments differ as opposed to \emph{system} variables which are the observed variables that form the feature set and target variable. The environment indicator $E$ described in the main paper can reasonably be viewed as a single context variable. Thus, invariance specs can be translated into the JCI framework. The specific method proposed in \citet{magliacane2018domain} considers a problem setup in which unlabeled target domain data is available. However, within \textsc{I-Spec} the assumption is that the unknown target environment will be drawn from the set of environments defined by an invariance spec $\langle \mathcal{P}, \mathbf{M} \rangle$. This stronger assumption is what allows \textsc{I-Spec} to be applied in settings in which no target environment data is available. Under this assumption it is straightforward to adapt the method of \citet{magliacane2018domain} to handle the input of \textsc{I-Spec}. However, the method of \citet{magliacane2018domain} searches only over stable conditional distributions.

Thus, both of these relevant dataset-driven methods are applicable to the same problems as \textsc{I-Spec}. However, they search over stable conditional distributions, which (under the assumptions of the \textsc{I-Spec} framework) consist of all the distributions (and only the distributions) that satisfy \citet[Theorem 30]{zhang2008causal} (which is sound and complete in PAGs). Then, by Lemma \ref{lemma:do-see} we get Corollary \ref{cor:subsume}, and we have that \textsc{I-Spec} subsumes existing dataset-driven methods in their ability to find stable distributions due to the additional search over stable interventional distributions.

\section{Learned PAG}\label{app:pag}
\begin{sidewaysfigure}[!h]
\centering
\includegraphics{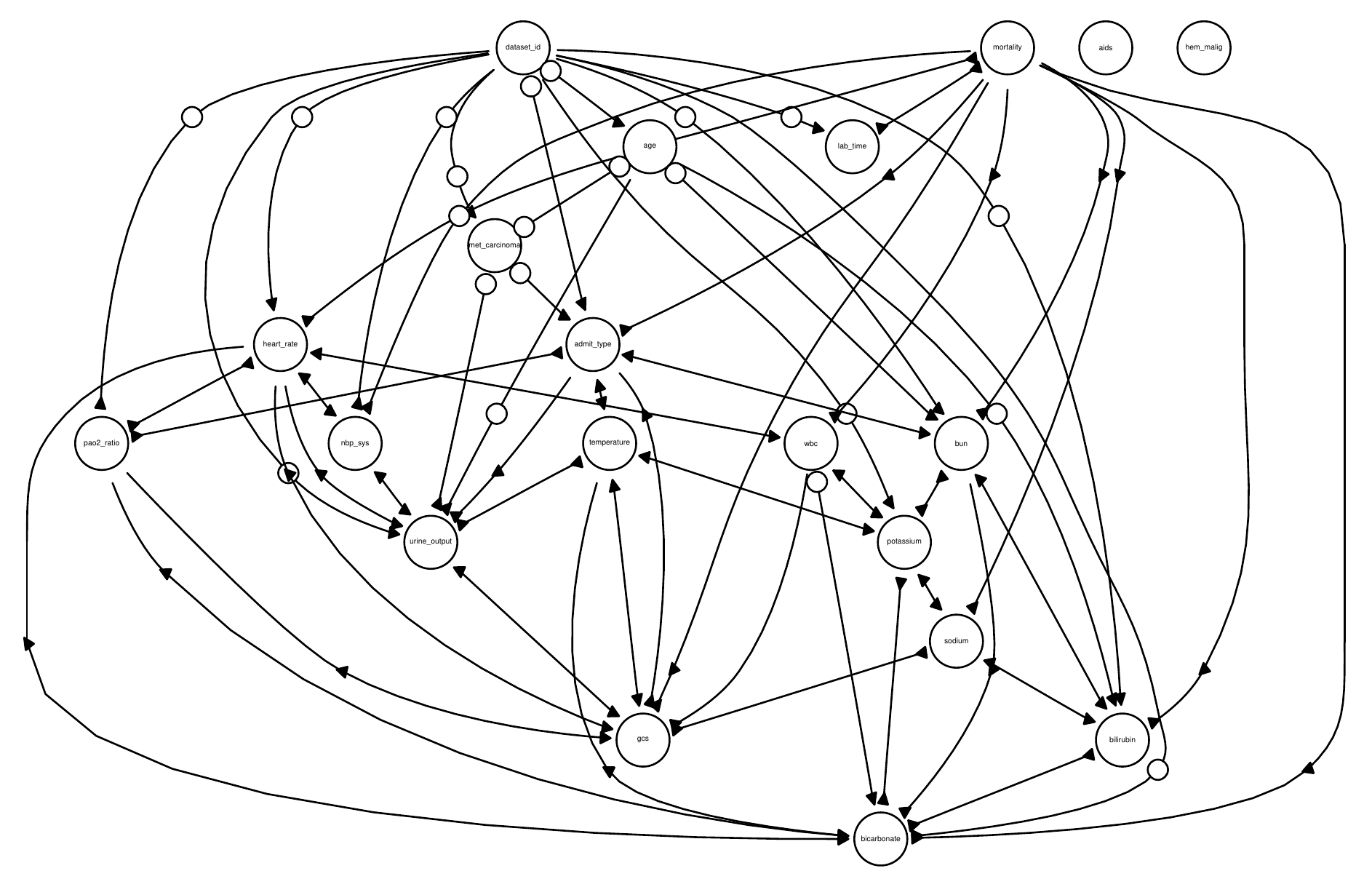}
\caption{The PAG learned using pooled FCI on the full, 3 hospital dataset. Because the dataset is mixed continuous and discrete, we use the Degenerate Gaussian Likelihood Ratio Test and set $\alpha=0.01$. Recall that bidirected edges $\leftrightarrow$ denote unobserved confounding, and that $\circ$ edge marks denote that there is at least one MAG in the equivalence class in which this mark is a head and at least one in which this is a tail.}
\label{fig:learned-pag}
\end{sidewaysfigure}

\clearpage

\section{Experimental Details}\label{app:exp}
\subsection{Simulated Experiment Details}\label{app:sim-exp}
We generated data according to the following linear Gaussian system:
\begin{align*}
    &X_3 \sim \mathcal{N}(0, 0.1^2) \\ 
    &U \sim \mathcal{N}(0, 0.1^2)\\
    &Y|U,X_3 \sim \mathcal{N}(0.5 X_3 + 5U, 0.1^2)\\
    &X_1|U \sim \mathcal{N}(\alpha U, 0.1^2)\\
    &X_2|Y,X_1 \sim \mathcal{N}(0.2 Y - X_1, 0.1^2) \\
\end{align*}

Different environments correspond to different values of the coefficient $\alpha$ in the structural equation for $X_1$. We generated 50,000 samples each from two source environments associated with $\alpha=4$ and $\alpha=8$. We pooled the data from these two environments to train all three models. We evaluated the three models in 100 test environments created by varying $\alpha$ on an evenly spaced grid from $-5$ to $17$, sampling 10,000 data points from each test environment.

We briefly note that $E[Y|do(X_1), X_2, X_3] = E[Y|X_2^*, X_3]$, where $X_2^* = X_2 - (-X_1)$ (e.g., $X_2$ with effect of $X_1$ removed). See \citet{subbaswamy2019hierarchy} for the equivalence of using the \emph{auxiliary variable} $X_2^*$ (a counterfactual variable) to the original interventional distribution. To compute $X_2^*$ we first fit a linear regression for the structural equation of $X_2$ to learn the coefficient of $X_1$ (which is -1). Then, using the estimated coefficient, we computed an estimate of $X_2^*$ before fitting the model $E[Y|X_2^*, X_3]$. Test environment $X_2^*$ values were computed using the coefficient learned from training data.

\subsection{Real Data Experiment Details:}
\paragraph{Data Cohort}
We construct the pooled dataset using de-identified measurements from patients who are admitted or transferred to the intensive care unit (ICU) of three hospitals in our institution’s network within from early 2016 to early 2018. We only consider patients who stayed in the ICU for longer than 24 hours and use data collected during the first 24 hours of their visit. We focus on the non-pediatric case, requiring all patients to be over 15 years old. For patients with multiple ICU encounters, we only consider data from their first encounter. These criteria result in a cohort of 24,787 patients. Mortality rates varied as follows: 7\% in H1, 10\% in H2, and 12\% in H3.

\paragraph{Data Features}
The target variable of our prediction model is Mortality, which is defined as an in-hospital death. We capture 12 physiologic features: Heart Rate, Systolic Blood Pressure, Temperature, Glasgow Coma Scale/Score (GCS), PaO$_{2}$/FiO$_{2}$, Blood Urea Nitrogen, Urine Output, Sodium, Potassium, Bicarbonate, Bilirubin, and White Blood Cell Count. We computed the worst value using the SAPS II criteria found in \citet{le1993new}. Furthermore, we consider age and three comorbidities: Metastatic cancer, Hematologic malignancy, and AIDS. SAPS II also makes use of the admission type (i.e., scheduled surgical, unscheduled surgical, or medical). To create a known shift, we simulate another healthcare process variable: time of day when lab measurements occur (i.e., morning $1$ or night $0$), such that mortality is correlated with morning measurements in Hospital 1, uncorrelated with measurement timing in Hospital 2, and correlated with night measurements in Hospital 3.

Specifically, we generated Lab Time $L$ as follows:
\begin{enumerate}
    \item $P(L=1|M=1,Hospital=1)=0.65$, $P(L=1|M=0,Hospital=1)=0.4$
    \item $P(L=1|M=1,Hospital=2)=0.5$, $P(L=1|M=0,Hospital=2)=0.5$
    \item $P(L=1|M=1,Hospital=3)=0.4$, $P(L=1|M=0,Hospital=3)=0.65$
\end{enumerate}

\paragraph{Imputation of missing values}
To account for the missing physiologic feature values, we impute our data via ``Last Observation Carried Forward'' (LOCF). If the feature value is missing from the patient’s first 24 hours, we impute it with the most recently recorded value prior to their ICU stay. Otherwise, we fill the missing value with the hospital-specific population mean.

\paragraph{Training}
We trained unregularized Logistic Regression models using ``classif.logreg'' in the 
\textbf{\textsf{R}} language's \texttt{mlr} package \citep{mlr}.

\end{document}